\documentclass[lettersize,journal]{IEEEtran}
\usepackage{amsmath,amsfonts}
\usepackage{algorithmic}
\usepackage{algorithm}
\usepackage{array}
\usepackage[caption=false,font=normalsize,labelfont=sf,textfont=sf]{subfig}
\usepackage{textcomp}
\usepackage{stfloats}
\usepackage{url}
\usepackage{verbatim}
\usepackage{graphicx}
\usepackage{cite}
\usepackage{multirow}
\usepackage{booktabs}
\usepackage{amsthm}
\newtheorem{lemma}{Lemma}
\usepackage{enumitem}
\usepackage[colorlinks, citecolor=blue]{hyperref}
\hypersetup{
	bookmarksnumbered=true,     
	bookmarksopen=true,         
	bookmarksopenlevel=1,       
	colorlinks=true,
	linkcolor=black            
}

\begin{document}

\title{\fontsize{19}{28}\selectfont Lightweight Task-Oriented Semantic Communication Empowered by Large-Scale AI Models}

\author{Chuanhong Liu, Caili Guo, \emph{Senior Member}, \emph{IEEE}, Yang Yang, Mingzhe Chen, \\and Tony Q. S. Quek, \emph{Fellow}, \emph{IEEE}
\vspace{-0.5cm}
\thanks{Copyright (c) 20xx IEEE. Personal use of this material is permitted. However, permission to use this material for any other purposes must be obtained from the IEEE by sending a request to pubs-permissions@ieee.org.}
\thanks{This work was supported in part by the National Natural Science Foundation of China (No. 62371070), and in part by Beijing Natural Science Foundation (No. L222043). (\textit{Corresponding author: Caili Guo})}
\thanks{C. Liu and C. Guo are with the Beijing Key Laboratory of Network System Architecture and Convergence, School of Information and Communication Engineering, Beijing University of Posts and Telecommunications, Beijing 100876, China (e-mail: 2016\_liuchuanhong@bupt.edu.cn; guocaili@bupt.edu.cn).}
\thanks{Y. Yang is with the Beijing Laboratory of Advanced Information Networks, School of Information and Communication Engineering, Beijing University of Posts and Telecommunications, Beijing 100876, China (e-mail: yangyang01@bupt.edu.cn).}
\thanks{M. Chen is with the Department of Electrical and Computer Engineering and Institute for Data Science and Computing, University of Miami, Coral Gables, FL, USA (e-mail: mingzhe.chen@miami.edu).}
\thanks{T. Q. S. Quek is with the Dept. of Information Systems Technology and Design, Singapore University of Technology and Design, Singapore, 487372 (e-mail: tonyquek@sutd.edu.sg).}
}

\maketitle
\begin{abstract}
Recent studies have focused on leveraging large-scale artificial intelligence (LAI) models to improve semantic representation and compression capabilities. However, the substantial computational demands of LAI models pose significant challenges for real-time communication scenarios. To address this, this paper proposes utilizing knowledge distillation (KD) techniques to extract and condense knowledge from LAI models, effectively reducing model complexity and computation latency. Nevertheless, the inherent complexity of LAI models leads to prolonged inference times during distillation, while their lack of channel awareness compromises the distillation performance. These limitations make standard KD methods unsuitable for task-oriented semantic communication scenarios. To address these issues, we propose a fast distillation method featuring a pre-stored compression mechanism that eliminates the need for repetitive inference, significantly improving efficiency. Furthermore, a channel adaptive module is incorporated to dynamically adjust the transmitted semantic information based on varying channel conditions, enhancing communication reliability and adaptability. In addition, an information bottleneck-based loss function is derived to guide the fast distillation process. Simulation results verify that the proposed scheme outperform baselines in term of task accuracy, model size, computation latency, and training data requirements.

\end{abstract}


\section{Introduction}
\IEEEPARstart{S}{emantic} communication (SemCom), a technique capable of notably enhancing communication efficiency and robustness, is now acknowledged as a foundational technology driving the progression of sixth-generation wireless networks\cite{framework_Yang}. Existing SemCom systems can be categorized into two main categories: source data reconstruction and task execution. In contrast to data reconstruction, task-oriented SemCom (TOSC) aims to transmit semantic-aware information and obtains the inference result in a mission-oriented fashion without recovering the original data, which is beneficial to meet the increasing requirements of complicated, diverse, and intelligent transmission.

Recently, various TOSC systems have been proposed to perform various downstream tasks\cite{liu2023task, Lee, ASC_liu, MU-DeepSC1}. The authors in \cite{liu2023task} proposed an explainable TOSC system for text sentiment analysis and question answering tasks based on semantic triplets. The authors in \cite{Lee, ASC_liu} focusing on image classification task, proposed various semantic encoding and compression schemes to reduce the transmission burden. The authors in \cite{MU-DeepSC1} developed multi-users SemCom system for visual question answering task. However, almost all current works suffer from limited semantic understanding capability stemming from constrained network structure and optimization strategies. These limitations significantly hamper the performance and broader applicability of SemCom.

Fortunately, there has been considerable advancement in the field of large-scale artificial intelligence (LAI) models, particularly in sophisticated transformer models encompassing billions of parameters. LAI models have achieved remarkable breakthroughs in various domains such as natural language understanding, image analysis, and video generation. These models offer numerous benefits, including comprehensive knowledge representation and precise semantic interpretation. Wireless networks are expected to increasingly provide diverse intelligent services with greater generality and efficiency, driven by the integration of LAI models\cite{maatouk2023large}. The authors in \cite{guo2023semantic} proposed a semantic importance-aware communication scheme using pre-trained large language models (LLMs). The pre-trained language model was utilized to quantify the semantic importance of data frames, based on which semantic importance-aware power allocation was investigated. 
The authors in \cite{jiang2023large} proposed a LAI model based SemCom framework for image data, in which segment anything model was designed as knowledge base to split the original image into different semantic segments. The authors in \cite{chen2024personalizing} establish a foundation model-based multi-user semantic communication framework via personalized federated parameter-efficient finetuning, in which each user equipment is provided with personalized semantic services in diverse downstream tasks.
The aforementioned works directly apply LAI models in the processes of semantic encoding or transmission. However, due to their vast number of parameters and high computational complexity, these models incur significant energy consumption and computational latency, even during inference. This makes them unsuitable for meeting the real-time and low-energy requirements of next-generation wireless communications. Therefore, there is an urgent need for research on lightweight semantic communication schemes that reduce communication complexity and fully exploit the strengths of LAI models.

To enable real-time TOSC for resource-constrained scenarios (e.g., Internet of vehicles), this paper proposes a novel lightweight TOSC scheme, where the semantic codec is distilled from LAI models. 
The primary contributions are summarized as follows:

\begin{itemize}
\item[$\bullet$] \textcolor{black}{\textit{To the best of the author's knowledge, this is the first work that integrates LAI model and TOSC through knowledge distillation.} By distilling LAI model into a streamline TOSC network, we can achieve both the advantages of large model's robust semantic representation capabilities, and the benefits from small model's low-power and low-latency attributes, striking a good balance between complexity and accuracy of semantic communication.}

\item[$\bullet$] \textcolor{black}{We propose a fast distillation method (FDM) that effectively transfers knowledge from LAI models to a lightweight TOSC model. This method significantly reduces distillation time and resource consumption while maintaining high performance. We then derive a loss function based on the information bottleneck (IB) principle to guide the fast distillation process, achieving the optimal semantic rate-distortion tradeoff. Consequently, the lightweight TOSC can significantly reduce the communication burden while maintaining the communication accuracy.}
 
\item[$\bullet$] \textcolor{black}{Additionally, to address the limitation of LAI models in accounting for wireless channel variations, we incorporate a channel adaptive module, which dynamically adjusts the weights of transmitted semantics in response to varying channel conditions, ensuring improved communication reliability and adaptability. Simulation results verify the superiority of the proposed scheme.}
\end{itemize}


\begin{figure*}[t]
	\begin{center}
		\includegraphics[width=0.8\linewidth]{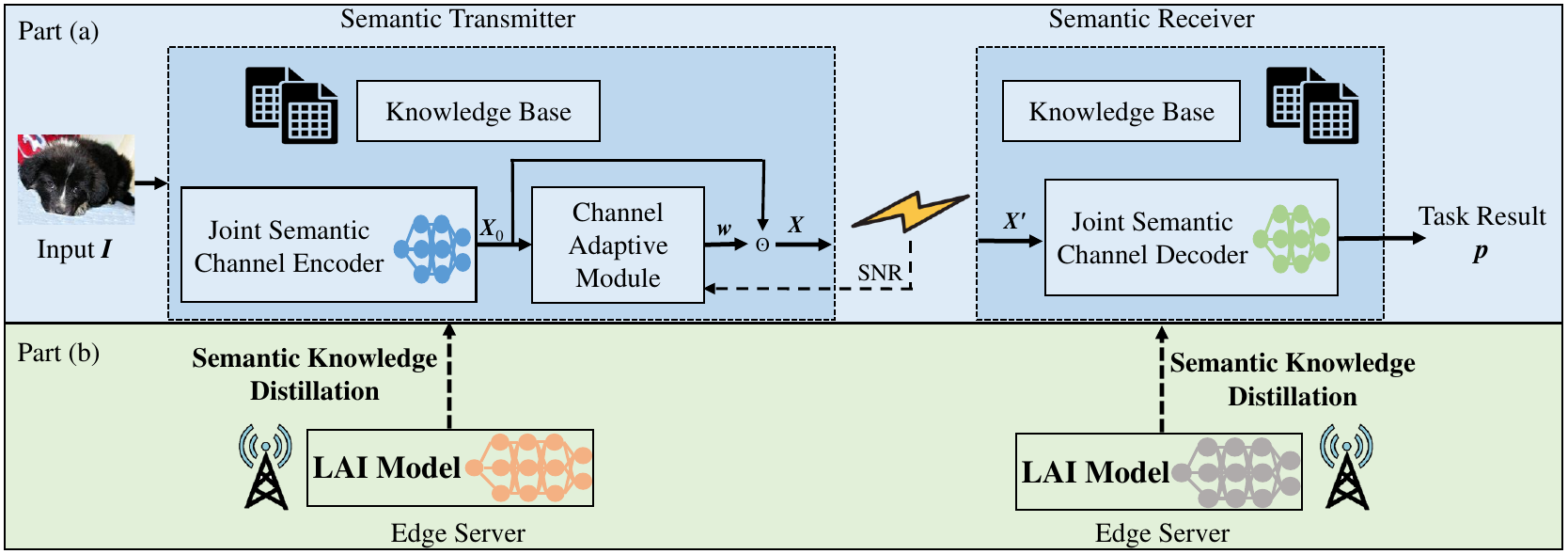}
	\end{center}
	\caption{\textcolor{black}{System model of lightweight task-oriented semantic communications empowered by large-scale artificial intelligence models.}}
	\label{fig:system}
	\vspace{-0.3cm}
\end{figure*}

\section{Proposed Distillation-based Lightweight TOSC Framework}
As shown in Fig. \ref{fig:system}, we consider a TOSC system that consists of a semantic transmitter, a semantic receiver, and two edge servers. Specifically, the transmitter consists of a joint semantic channel encoder to extract the semantic features, and a channel adaptive module to combat the channel noise. The receiver is composited with a joint semantic channel decoder for semantics recovery and task execution. The LAI models are deployed at edge servers to guide the training of joint semantic channel encoder and decoder, i.e., part (b) in Fig. \ref{fig:system}.

We first detail the end-to-end (E2E) communication process, as shown in part (a) of Fig. \ref{fig:system}. At the transmitter, the joint semantic channel encoder is first utilized to extract the semantics from source data $\boldsymbol{I}$, which can be denoted by
\begin{eqnarray}\label{SE}
	{\boldsymbol {X}_0} = {R_{\boldsymbol{\alpha}}}({\boldsymbol {I}}),
\end{eqnarray}
where ${R_{\boldsymbol {\alpha}}}(\cdot)$ denotes the joint semantic channel encoder with parameter vector $\boldsymbol{\alpha}$.

The extracted semantics ${\boldsymbol {X}_0}$ is then fed into a channel adaptive module, which utilizes both the extracted semantics and the channel signal-to-noise ratio (SNR) as inputs. The module outputs semantic weights\cite{channeladaptive}, which can be represented by
\begin{align}\label{}
{\boldsymbol{w}} = {\Phi _{\boldsymbol{\theta }}}\left( {{\boldsymbol{X}_0}, s} \right)
\end{align}
where ${\Phi _{\boldsymbol{\theta }}}\left( { \cdot } \right)$ is the channel adaptive module with trainable parameter vector $\boldsymbol{\theta}$ and $s$ is the channel SNR.

Following this, the weighted semantics can be denoted by
\begin{align}\label{}
{\boldsymbol{X}} = {\boldsymbol{w}} \odot {\boldsymbol{X}_0}
\end{align}
where $\odot$ is channel-wise multiplication.

The received noise-corrupted signal at the receiver can be expressed as
\begin{eqnarray}\label{Channel}
{\boldsymbol{X}'} = {\boldsymbol H}{\boldsymbol{X}} + \boldsymbol{n},
\end{eqnarray}
where ${\boldsymbol H}$ is the channel gain, $\boldsymbol{n}\sim{{\cal C}{\cal N}}(0,{\boldsymbol{\sigma }}^2{\boldsymbol I})$ is the independent and identically distributed complex Gaussian noise, and ${\boldsymbol{\sigma }}^2$ indicates the noise variance.

At the receiver, the received semantics $\boldsymbol{X}'$ is proceeded via joint semantic channel decoder to complete the task, e.g. image classification, which can be expressed as
\begin{eqnarray}\label{SD}
	{\boldsymbol{p}} = {R_{\boldsymbol{\mu}}^{ - 1}}({\boldsymbol{{{\boldsymbol{X}}'}}}),
\end{eqnarray}
where ${\boldsymbol{p}}$ is the task result, and ${R_{\boldsymbol{\mu}}^{ - 1}}(\cdot)$ denotes the joint semantic channel decoder with parameter vector ${\boldsymbol{\mu}}$. 

\textcolor{black}{To effectively utilize LAI for guiding the training of the semantic encoder and decoder, it is crucial to design a lightweight semantic codec that can be efficiently implemented by both the transmitter and receiver, given their resource constraints. In this work, we focus on leveraging KD to streamline the model.} 
While LAI models bring substantial benefits, two main challenges arise during the conventional knowledge distillation process, necessitating improvements to conventional methods. 
\begin{enumerate}[label=(\roman*)]
    \item \textcolor{black}{Since conventional knowledge distillation necessitates repeated inference using the LAI model, a considerable amount of computing resources must be allocated to process training data through the LAI model during each iteration. This process is both inefficient and costly. Moreover, in practical semantic communication systems, model updates are often required due to dynamic changes in the environment or communication channels. The high computational overhead of conventional distillation makes frequent model retraining impractical, limiting the adaptability of the system to evolving conditions.}
    \item The pretrained LAI model lacks channel awareness, which can significantly impact the performance of the distilled lightweight model, especially in dynamic wireless environments with rapidly changing channel conditions.
\end{enumerate}


\textcolor{black}{In the subsequent section, we propose a novel distillation framework to overcome these two challenges.}

\section{Proposed KD-Based Lightweight TOSC}
\label{algorithm}

\subsection{Proposed Distillation Framework} 


To address challenge (i), we propose a fast distillation method, which adopts a pre-stored compression mechanism that computes and stores the logits of the LAI model in advance, so as to eliminate the need for repeated inference with the large-scale teacher model during the distillation process. Furthermore, to alleviate the storage burden of extensively pre-stored logits, we adopt a logits compression method. This method involves selecting and retaining only the top-$K$ largest values, denoted as $\left\{ {{y_{i\left( k \right)}}} \right\}_{k = 1}^K \in {\boldsymbol{y}}$, and their corresponding indices, $\left\{ {i\left( k \right)} \right\}_{k = 1}^K$, from each logit vector ${\boldsymbol{y}}$. During the distillation phase, these compressed logits are then utilized through logit smoothing\cite{TinyViT}. Mathematically, it can be represented by
\begin{align}
{\hat y_n} = \left\{ {\begin{array}{*{20}{c}}
{{y_{i\left( k \right)}},}&{n = i\left( k \right),k = 1,2, \ldots ,K,}\\
{\frac{{1 - \sum\limits_{k = 1}^K {{y_{i\left( k \right)}}} }}{{C - K}},}&{n \ne i\left( k \right),k = 1,2, \ldots ,K,}
\end{array}} \right.
\end{align}
where ${\boldsymbol{\hat y}} = \left[ {{{\hat y}_1},{{\hat y}_2}, \cdots ,{{\hat y}_n}, \cdots ,{{\hat y}_C}} \right]$ is the smoothed logits for distillation and $C$ is the dimension of logit ${\boldsymbol{y}}$. When the compression factor $K$ is small, i.e., $K \ll C$, it can reduce logits’ storage by orders of magnitude. Such a pre-stored compression mechanism markedly diminishes the computational load and energy expenditure, making the distillation process more efficient and viable for lightweight TOSC model. 

\subsection{IB-Based Distillation Loss}
\textcolor{black}{In this subsection, we delve into deriving the loss function for the fast distillation method. Conventional knowledge distillation loss functions, such as mean squared error or Kullback-Leibler (KL) divergence, are primarily designed to align the outputs of teacher and student models. While effective in many contexts, these methods often fall short in promoting efficient compression. Specifically, the distilled lightweight model may produce semantic features with considerable redundancy, as traditional loss functions lack mechanisms to enforce compact and efficient data representations. This limitation can result in semantic inefficiencies, with unnecessary information being retained in the distilled model. To address this issue, we incorporate IB theory \cite{Sun_AIB} into the distillation process, utilizing mutual information (MI) to quantify and control the trade-off between knowledge retention and compression. Unlike conventional IB-based approaches that optimize end-to-end model training\cite{Shao_IB, 10579852,10738311}, our method applies IB principles within a knowledge distillation framework, guiding the transfer of essential information from the large model to the lightweight model. This ensures that the student model not only mimics the teacher’s outputs but also captures the most relevant semantic features in a compact form. By introducing MI-based constraints, our approach enhances the effectiveness of knowledge transfer while reducing unnecessary information, ultimately improving both distillation efficiency and semantic communication performance.} The formulation of IB can be formulated as
\begin{eqnarray}\label{IB}
\mathcal{L} = I\left( {\boldsymbol{I};\boldsymbol{X}} \right) - \beta I\left( {\boldsymbol{X}';\boldsymbol{Y}} \right),
\end{eqnarray}
where $\boldsymbol{I}$ is the source data, and $\boldsymbol{X}$ is the transmitted semantics. $\boldsymbol{X}'$ is the received semantics and $\boldsymbol{Y}$ is the inference label. $\beta > 0$ is a hyperparameter for balancing compression and task performance. The first term is the rate term, encouraging $\boldsymbol{X}$ to compress information related to $\boldsymbol{I}$.
The second term is the distortion term, encouraging $\boldsymbol{X}'$ to predict $\boldsymbol{Y}$, which can be denoted by
\begin{align}\label{IB2}
I\left( {\boldsymbol{X}';\boldsymbol{Y}} \right) = H(\boldsymbol{Y}) - H(\boldsymbol{Y}\left| \boldsymbol{X}' \right.) \implies  - H(\boldsymbol{Y}\left| \boldsymbol{X}' \right.),
\end{align}
where $H(\cdot)$ is the entropy and $H(\boldsymbol{Y})$ is ignored since it is a constant. To reduce dependence on true labels, which are often difficult or impractical to obtain in real-world scenarios, we leverage the outputs of LAI model as pseudo-labels\footnote{In the proposed fast distillation method, $\boldsymbol{Y}$ is the pre-stored logits of LAI model.}. This strategy enables the effective utilization of unlabeled data, offering a label-free approach that maximizes the value of unlabeled datasets. However, (\ref{IB}) is mathematically intractable due to the unknown joint and marginal distributions.

To solve this problem, we first derive the upper bound of $I\left( {\boldsymbol{I};\boldsymbol{X}} \right)$, which obtained via the following lemma.
\begin{lemma}
    The upper bound of $I\left( {\boldsymbol{I};\boldsymbol{X}} \right)$ is
\begin{align}\label{IB3}
    I_{\rm {up}}\left( {\boldsymbol{I};\boldsymbol{X}} \right) = {{\mathbb{E}}_{p\left( {\boldsymbol{i},\boldsymbol{x}} \right)}}\left[ {\log {p_{\boldsymbol{\alpha}} }\left( {\boldsymbol{x}\left| \boldsymbol{i} \right.} \right)} \right] - \nonumber\\{{\mathbb{E}}_{p\left( \boldsymbol{x} \right)}}{{\mathbb{E}}_{p\left( \boldsymbol{i} \right)}}\left[ {\log {p_{\boldsymbol{\alpha}} }\left( {\boldsymbol{x}\left| \boldsymbol{i} \right.} \right)} \right]
\end{align}
\end{lemma}

\begin{proof}
The proof is omitted due to the page limitation.
\end{proof}
\textcolor{black}{Then, we can use Monte Carlo method to estimate $I_{\rm {up}}\left( {\boldsymbol{I};\boldsymbol{X}} \right)$, which can be denoted by
\begin{align}\label{loss1}
{\hat I}_{\rm {up}}\left( {\boldsymbol{I};\boldsymbol{X}} \right)&=\frac{1}{N}\sum\limits_{n = 1}^N \log {p_{\boldsymbol{\alpha}} }\left( {{\boldsymbol{x}_n}\left| {{\boldsymbol{i}_n}} \right.} \right) \nonumber \\&- \frac{1}{{{N^2}}}\sum\limits_{n = 1}^N {\sum\limits_{j = 1}^N {\log {p_{\boldsymbol{\alpha}} }\left( {{\boldsymbol{x}_n}\left| {{\boldsymbol{i}_j}} \right.} \right)} },
\end{align}
where $N$ is the number of samples, and ${\boldsymbol{i}_n}$ and ${\boldsymbol{x}_n}$ are the $n$-th input and $n$-th semantics, respectively.}
Next, we employ the variational information bottleneck (VIB) approach\cite{Sun_AIB} to derive a tractable bound of the distortion term, $I\left( {\boldsymbol{X}';\boldsymbol{Y}} \right)$. Specifically, ${q_{\boldsymbol{\mu}}(\boldsymbol{y}\left| \boldsymbol{x}'\right.)}$ is used to approximate the true distribution $p(\boldsymbol{y}\left| \boldsymbol{x}'\right.)$. Based on the non-negative property of KL divergence, we can obtain
\begin{align}\label{}
\mathbb{E}_{p(\boldsymbol{y} \mid \boldsymbol{x}')} \left[ \log p(\boldsymbol{y} \mid \boldsymbol{x}') \right] \geq \mathbb{E}_{p(\boldsymbol{y} \mid \boldsymbol{x}')} \left[ \log q_{\boldsymbol{\mu}}(\boldsymbol{y} \mid \boldsymbol{x}') \right].
\end{align}
Then, the variational upper bound for the conditional entropy of $\boldsymbol{Y}$ given $\boldsymbol{X}'$ can be denoted by
\begin{align}\label{loss2}
H(\boldsymbol{Y} \mid \boldsymbol{X}') \leq \mathbb{E}_{p(\boldsymbol{y}, \boldsymbol{x}')} \left[ -\log q_{\boldsymbol{\mu}}(\boldsymbol{y} \mid \boldsymbol{x}')\right].
\end{align}
\textcolor{black}{Particularly, the negative expectation term in the FDM can be achieved by the soft cross-entropy and Monte Carlo sampling. Therefore, by incorporating (\ref{loss1}) and (\ref{loss2}), the IB-based loss function of the FDM can be expressed as
\begin{align}\label{loss_f}
\mathcal{L} &= {\hat I}_{\rm {up}}\left( {\boldsymbol{I};\boldsymbol{X}} \right) - \beta \frac{1}{N} \sum_{n = 1}^N \boldsymbol{y}_n \log \left( S_{\left(\boldsymbol{\alpha}, \boldsymbol{\mu} \right)}\left( \boldsymbol{i}_n \right) \right),
\end{align}
where $\boldsymbol{i}_n$ is the $n$-th input and $S_{\left(\boldsymbol{\alpha}, \boldsymbol{\mu} \right)}\left( \cdot \right)$ encapsulates the operations of the lightweight TOSC network, integrating both the encoder, ${R_{\boldsymbol {\alpha}}}(\cdot)$, and decoder, ${R_{\boldsymbol{\mu}}^{ - 1}}(\cdot)$. (\ref{loss_f}) is a tractable and differential form of IB-based loss function and is used for training.}

\subsection{Channel Adaptive Module}
\textcolor{black}{To address challenge (ii), we leverage the channel adaptive module\cite{channeladaptive} within our system, enabling it to dynamically incorporate real-time channel information into the semantic representation. This integration significantly enhances the system's adaptability and overall performance, especially under fluctuating communication conditions. Specifically, the channel adaptive module is composed of two dense layers and operates on the semantic features generated by the lightweight encoder, denoted as $\boldsymbol{X}_0 \in {\mathbb{R}^{C \times H \times W}}$. $\boldsymbol{X}_0$ is first transformed into a $C$-dimensional vector via average pooling along the spatial dimensions. This pooling step condenses spatial information into a compact representation in the feature channel domain, capturing essential characteristics. The resulting vector is then concatenated with the current channel SNR, encapsulating both semantic and channel state information. Next, the concatenated vector is passed through two fully connected dense layers. These layers are designed to compute semantic weights conditioned on the channel SNR, effectively mapping the combined input into a set of importance weights. These weights dynamically emphasize the semantics most relevant for maintaining reliable communication under varying SNR scenarios. Finally, the computed weights are applied to $\boldsymbol{X}_0$ through channel-wise multiplication, selectively scaling semantic features. This process allows the system to adaptively prioritize critical semantic elements that enhance communication robustness, minimizing the impact of adverse channel conditions. By dynamically aligning the semantic representation with real-time channel states, this module plays a crucial role in ensuring the overall reliability and efficiency of the proposed system.}

\section{Experiment Results and Analysis}
\subsection{Simulation Setup}
1) \emph{Dataset}: For evaluation, we employ the widely used Imagenet-1k dataset\cite{imagenet}, which includes 1000 object categories. Top-1 accuracy serves as the primary performance metric to assess the effectiveness of the proposed scheme.

\textcolor{black}{2) \emph{Baselines}: Since this is the first work focusing on LAI empowered image classification-oriented SemCom, inspired by \cite{jiang2023large}, we compare the proposed scheme (labeled as L-TOSC-KD) with the following baselines. 1) LAI model based TOSC: which directly uses LAI model for encoding and decoding, labeled as LA-TOSC. In this work, CLIP-ViT-L/14\cite{CLIP} is chosen as the LAI model for knowledge distillation. Developed by OpenAI, CLIP-ViT-L/14 is a state-of-the-art multimodal model that utilizes a vision transformer (ViT) with 14 layers to facilitate robust cross-modal learning and support applications such as zero-shot classification. 2) Lightweight TOSC without distillation: which trains the same lightweight TOSC network in a E2E manner without knowledge distillation, labeled as L-TOSC. 3) Conventional TOSC schemes: that use the backbone of Resnet and EfficientNet for semantic codec, labeled as Res-TOSC\cite{Lee} and Eff-TOSC\cite{ASC_liu}, respectively. To guarantee a fair comparison, we ensure that the proposed scheme maintains parameters and complexity that are comparable to or lower than those of the baselines.}

3) \emph{Structure of Semantic Codec}: The lightweight TOSC's encoder comprises a patch embedding block followed by a four stage ViT encoder, and the patch embedding block is constructed by two convolutional layers, each with a kernel size of 3, stride of 2, and padding of 1. The lightweight TOSC's decoder consists of two dense layers. The first layer has 512 units, followed by a second layer with 1000 units, which outputs the results of the task.

The experiments are conducted on a computer with Ubuntu 16.04 + CUDA 11.0. The selected deep learning framework for these experiments is PyTorch, and the hardware utilized includes an NVIDIA V100 GPU.

\subsection{Performance Comparison}
\begin{figure}[t]
\setlength{\abovecaptionskip}{1pt}
	\begin{center}
		\includegraphics[width=1\linewidth]{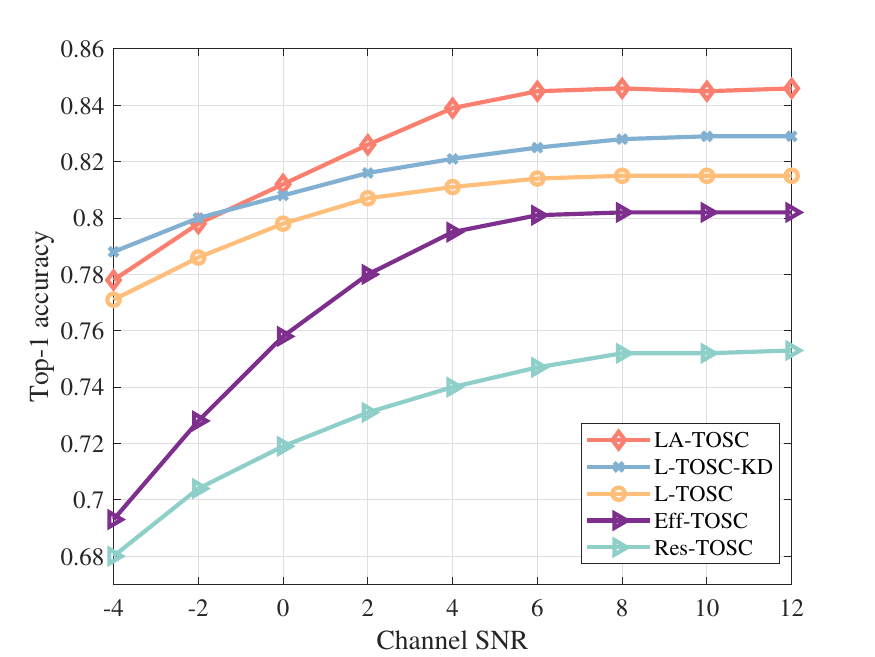}
	\end{center}
	\caption{Accuracy of different schemes versus SNRs.}
	\label{fig:AWGN}
\vspace{-0.3cm}
\end{figure}

{Fig. \ref{fig:AWGN} illustrates the performance comparison between the proposed L-TOSC-KD scheme and the baselines in terms of SNR. During the training process, all methods except LA-TOSC are trained with SNR values randomly sampled between $-$5 dB and 15 dB, and tested between $-$4 dB and 12 dB. As depicted in this figure, the LA-TOSC scheme achieves superior performance for SNRs greater than 0 dB, attributable to its intricate network architecture and extensive pretraining on a vast dataset. However, the enhanced performance comes at the expense of significant computational resources, rendering it less suitable for real-time communication scenarios. In addition, the proposed L-TOSC-KD outperform both L-TOSC and conventional TOSC schemes across the entire range of SNR values. Specifically, at an SNR of 0 dB, L-TOSC-KD shows a 1.3\%, 6.6\% and 12.4\% increase in accuracy compared to L-TOSC, Eff-TOSC, and Res-TOSC, respectively. This improvement is attributed to the scheme's ability to assimilate knowledge from the LAI model, significantly boosting its semantic coding capabilities. Intriguingly, at SNRs below 0 dB, the L-TOSC-KD model demonstrates superior performance compared to the LA-TOSC. This is due to the channel adaptive module, which incorporates the channel state during the fast distillation process, highlighting its advantage in scenarios where the LAI model lacks prior knowledge of the wireless channels.

\begin{figure}[t]
\setlength{\abovecaptionskip}{1pt}
	\begin{center}
		\includegraphics[width=\linewidth]{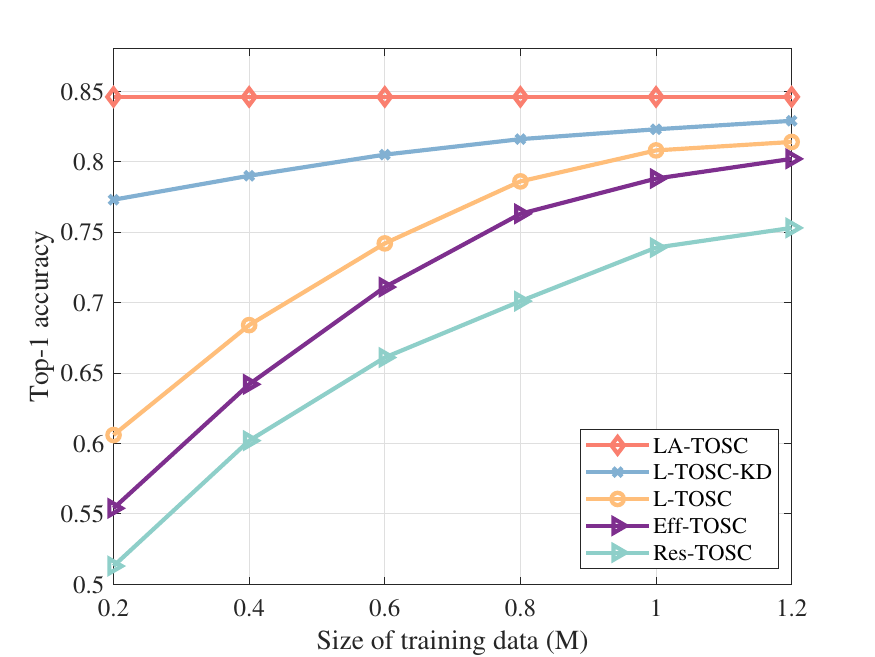}
	\end{center}
	\caption{Accuracy of different schemes versus training data size.}
	\label{fig:size}
\vspace{-0.3cm}
\end{figure}

\textcolor{black}{Fig. \ref{fig:size} shows the relationship between the accuracy and the training data size over the proposed L-TOSC-KD and baselines. It should be noted that the we maintain consistent data augmentation methods between the proposed scheme and the comparison baselines to ensure a fair comparison. The LA-TOSC, utilizing a pre-trained LAI model as its semantic encoder, does not require additional training, resulting in its performance being depicted as a constant line across varying sizes of training data. In addition, it is evident that the proposed L-TOSC-KD scheme outperforms all the baselines, even with limited training data. Specifically, with a training data size of 0.6 million, L-TOSC-KD achieves performance gains of 15.5\%, 26.5\% and 31.2\% over L-TOSC, Eff-TOSC, and Res-TOSC, respectively. This enhancement is attributed to the knowledge distillation process, which enables the L-TOSC-KD to emulate the behavior of the larger model and provides additional knowledge for more effective training. Additionally, the inferior performance of Eff-TOSC and Res-TOSC confirms the superior architecture of the proposed lightweight TOSC scheme.}

\renewcommand\arraystretch{1.2}
\begin{table}[t]
\centering
\small
\caption{The Complexity Analysis With The Model Size, Number of Parameters and Computation Latency.}
\begin{tabular}{cccc}
\toprule
Model     & Model Size & Parameters & Latency \\
\midrule
LA-TOSC   & 1164.7 MB & 305,311,208   & 93.78 ms  \\
L-TOSC-KD & 41.9 MB   & 10,996,972    & 18.06 ms   \\
L-TOSC    & 41.9 MB   & 10,996,972    & 18.06 ms  \\
Eff-TOSC  & 45.9 MB   & 12,036,884    & 23.53 ms  \\
Res-TOSC  & 97.8 MB   & 25,636,712    & 37.86 ms  \\
\toprule
\end{tabular}
\label{tab:complexity}
\vspace{-0.5cm}
\end{table}

\textcolor{black}{In Table \ref{tab:complexity}, we conduct the complexity analysis for the proposed scheme and baselines in terms of the model size, the number of parameters, and the average computation latency per image. Thanks to the proposed FDM, the proposed L-TOSC-KD demonstrates both the smallest size and the lowest parameter count among all the compared schemes. Specifically, the L-TOSC-KD's model size is 41.9 MB, constituting merely 3.6\% of the LA-TOSC's size and 42.8\% of that of Res-TOSC. Additionally, the LA-TOSC exhibits the longest transmission time per image, approximately 94 ms, attributable to its highly complex architecture. In contrast, L-TOSC-KD and L-TOSC demonstrate significantly reduced computation latency, achieving an inference speedup of over 5×. However, this reduction in model size and computational cost comes with a slight trade-off in task performance. While the knowledge distillation-enhanced L-TOSC-KD greatly improves efficiency, some minor performance degradation is observed compared to LA-TOSC, as the distilled model may not fully retain all the knowledge from the larger LAI model. Nonetheless, given the substantial reduction in computation time and resource consumption, L-TOSC-KD remains a favorable choice for real-time semantic communication scenarios, where inference speed and efficiency are crucial.}

\renewcommand\arraystretch{1.3}
\begin{table}[t]
\centering
\small
\caption{Ablation Study for With and Without FDM.}
\begin{tabular}{ccc}
\toprule
Method           & Storage Cost & Training Time \\
\midrule
with FDM    & 16 GB        & 26 hours      \\
without FDM & 139 GB \textcolor{blue}{(8.7$\times$)}       & 63 hours \textcolor{blue}{(2.4$\times$)}    \\
\toprule
\end{tabular}
\label{tab:ablation}
\vspace{-0.5cm}
\end{table}

We further conduct an ablation study to verify the effectiveness of the proposed FDM, setting $K$ as 10, with results displayed in Table \ref{tab:ablation}. ``without FDM'' refers to the method that uses LAI model for repeated inference during the lightweight TOSC distillation process. Notably, the storage cost for the method without FDM is 8.7 times higher than our proposed method. This discrepancy arises because our method retains only the essential logits, eliminating the need to store voluminous original image data ($\approx138$ GB) and the extensive parameters of the teacher model. Additionally, since FDM merely requires loading the teacher logits from a hard disk during training, it is 2.4 times faster than the method without FDM, which significantly improves the distillation efficiency.

\section{Conclusion}
\textcolor{black}{This paper presented a lightweight TOSC scheme, augmented by LAI model via KD. Specifically, a fast distillation method is proposed to improve the KD efficiency and performance. To address the robustness limitations of LAI models, a channel adaptive module is adopted to adjust the weights of transmitted semantic features to enhance overall system resilience. Furthermore, an IB-based distillation loss is derived to guide the training of the lightweight semantic codec. Simulation results verified the significant superiority of the proposed lightweight TOSC scheme. This paper laid a foundation for more intricate and advanced works of LAI model based TOSC.}


 





\bibliographystyle{IEEEbib}
\nocite{*}\bibliography{reference}
\end{document}